\documentclass{article}

\usepackage[]{preprint_arxiv}

\usepackage{natbib}
\usepackage[utf8]{inputenc} 
\usepackage[T1]{fontenc}    
\usepackage{hyperref}       
\usepackage{url}            
\usepackage{booktabs}       
\usepackage{amsfonts}       
\usepackage{nicefrac}       
\usepackage{microtype}      

\usepackage{multirow}
\usepackage{mathtools}
\usepackage{times}
\usepackage{graphicx}
\usepackage{tikz}
\usepackage{tkz-graph}
\usetikzlibrary{patterns}

\usepackage{amsthm}
\usepackage{amssymb}
\usepackage{bbm}
\usepackage{bm}
\usepackage{commath}

\usepackage[title,titletoc]{appendix}
\usepackage{color}

\definecolor{Kred}{RGB}{175, 0, 0}
\definecolor{Kyellow}{RGB}{204,204,0}
\definecolor{Kblue}{RGB}{76, 161, 245}
\definecolor{Kgreen}{RGB}{76, 159, 119}

\newcommand{\Scope}{\mathrm{Scope}}
\newcommand{\VCtwo}{\{0,1\}^2}
\newcommand{\VCtwoDim}{\{0,1\}^2\text{-Dim}}
\newcommand{\MaxVCtwo}{\text{max-}\{0,1\}^2}
\newcommand{\MaxVCtwoDim}{\text{max-}\{0,1\}^2\text{-Dim}}
\newcommand{\duv}{d^{(0)}_{u,v}}
\newcommand{\VCDim}{\text{VC-Dim}}

\newcommand{\Mfrak}{\mathfrak{M}}

\newcommand{\fh}{\hat{f}}
\newcommand{\thetah}{\widehat{\theta}}
\newcommand{\Bh}{\widehat{B}}

\newcommand{\fs}{f^*}

\newcommand{\Ind}[1]{\ensuremath{\mathbbm{1} \hspace{-0.03in} \left[#1\right]}}                     

\newcommand{\Norm}[1]{\ensuremath{\lVert #1 \rVert}}                  

\newcommand{\InNorm}[1]{{\left\vert\kern-0.2ex\left\vert\kern-0.2ex\left\vert #1 
    \right\vert\kern-0.2ex\right\vert\kern-0.2ex\right\vert}}                    

\newcommand{\InNormII}[1]{{\left\vert\kern-0.2ex\left\vert\kern-0.2ex\left\vert #1 
    \right\vert\kern-0.2ex\right\vert\kern-0.2ex\right\vert}_2}                    

\newcommand{\InNormInfty}[1]{{\left\vert\kern-0.2ex\left\vert\kern-0.2ex\left\vert #1 
    \right\vert\kern-0.2ex\right\vert\kern-0.2ex\right\vert}_{\infty}}           

\newcommand{\iid}{i.i.d.~}                                                        

\newcommand{\defeq}{\overset{\mathrm{def}}{=}}                                   

\newtheorem{definition}{Definition}
\newtheorem{proposition}{Proposition}

\newtheorem{theorem}{Theorem}

\newcommand{\bE}{\mathbb{E}}
\newcommand{\bP}{\mathbb{P}}
\DeclareMathOperator*{\E}{\mathbb{E}}
\let\P\undefined
\DeclareMathOperator*{\P}{\mathbb{P}}

\newcommand{\tand}{\mathrm{and}}
\newcommand{\tor}{\mathrm{or}}

\DeclareMathOperator*{\argmax}{arg\,max}
\DeclareMathOperator*{\argmin}{arg\,min}

\newcommand{\figleft}{{\em (Left) }}
\newcommand{\figcenter}{{\em (Center) }}
\newcommand{\figright}{{\em (Right) }}

\def\1{\bm{1}}

\DeclareMathAlphabet{\mathsfit}{\encodingdefault}{\sfdefault}{m}{sl}
\SetMathAlphabet{\mathsfit}{bold}{\encodingdefault}{\sfdefault}{bx}{n}

\def\gA{{\mathcal{A}}}

\def\gD{{\mathcal{D}}}

\def\gF{{\mathcal{F}}}
\def\gG{{\mathcal{G}}}
\def\gH{{\mathcal{H}}}

\def\gO{{\mathcal{O}}}
\def\gP{{\mathcal{P}}}

\def\gR{{\mathcal{R}}}

\def\gV{{\mathcal{V}}}

\def\gX{{\mathcal{X}}}
\def\gY{{\mathcal{Y}}}
\def\gZ{{\mathcal{Z}}}

\def\sD{{\mathbb{D}}}

\def\sR{{\mathbb{R}}}

\allowdisplaybreaks 

\author{%
		Kevin Bello\\
		Department of Computer Science\\
		Purdue Univeristy\\
		West Lafayette, IN 47906, USA \\
		\texttt{kbellome@purdue.edu} \\
	\And
		Asish Ghoshal\\
		Department of Computer Science\\
		Purdue Univeristy\\
		West Lafayette, IN 47906, USA \\
		\texttt{aghoshal@purdue.edu} \\
	\And
		Jean Honorio\\
		Department of Computer Science\\
		Purdue Univeristy\\
		West Lafayette, IN 47906, USA \\
		\texttt{jhonorio@purdue.edu} \\
}

\title{Minimax bounds for structured prediction}

\begin{document}

\maketitle

\begin{abstract}
Structured prediction can be considered as a generalization of many standard supervised learning tasks, and is usually thought as a simultaneous prediction of multiple labels. 
One standard approach is to maximize a score function on the space of labels, which decomposes as a sum of unary and pairwise potentials, each depending on one or two specific labels, respectively.
For this approach, several learning and inference algorithms have been proposed over the years, ranging from exact to approximate methods while balancing the computational complexity.
However, in contrast to binary and multiclass classification, results on the necessary number of samples for achieving learning is still limited, even for a specific family of predictors such as factor graphs.
In this work, we provide minimax bounds for a class of factor-graph inference models for structured prediction.
That is, we characterize the necessary sample complexity for any conceivable algorithm to achieve learning of factor-graph predictors.
\end{abstract}


\section{Introduction}
Structured prediction has been continuously used over the years in multiple domains such as computer vision, natural language processing, and computational biology. 
Key examples of structured prediction problems include image segmentation, dependency parsing, part-of-speech tagging, named entity recognition, machine translation and protein folding. 
In this setting, the input $x$ is some observation, e.g., social network, an image, a sentence. 
The output is a labeling $y$, e.g., an assignment of each individual of a social network to a cluster, or an assignment of each pixel in the image to foreground or background, or an acyclic graph as in dependency parsing.
A property common to these tasks is that, in each case, the natural loss function admits a decomposition along the output substructures.
Thus, a common approach to structured prediction is to exploit local features to infer the global structure.
For instance, one could include a feature that encourages two individuals of a social network to be assigned to different clusters whenever there is a strong disagreement in opinions about a particular subject.
Then, one can define a posterior distribution over the set of possible labelings conditioned on the input.

The output structure and corresponding loss function make these problems significantly different from the (unstructured) binary or multiclass classification problems extensively studied in learning theory.
Some classical algorithms for learning the parameters of the model include conditional random fields  \citep{lafferty2001conditional}, structured support vector machines \citep{Taskar03,tsochantaridis2005large,Altun03}, kernel-regression algorithm \citep{cortes2007general}, search-based structured prediction \citep{daume2009search}.
More recently, deep learning algorithms have been developed for specific tasks such as image annotation \citep{vinyals2015show}, part-of-speech-tagging \citep{jurafsky2014speech,vinyals2015grammar}, and machine translation \citep{zhang2008structured}.

However, in contrast to the several algorithms developed, there have been relatively few studies devoted to the theoretical understanding of structured prediction. 
From the few theoretical literature, the most studied aspect has been the generalization error bounds.
\citep{cortes2014ensemble,collins2004parameter,taskar2004max} provided learning guarantees that hold primarily for losses such as the Hamming loss and apply to specific factor graph models. 
\citep{McAllester07,honorio2016,bello2018learning,ghoshal2018learning} provide PAC-Bayesian guarantees for arbitrary losses through the analysis of randomized algorithms using count-based hypotheses.
Literature on lower bounding the sample complexity for structure prediction is scarcer even for specific hypothesis classes of losses.
Information-theoretic bounds have been studied in the context of binary graphical models \citep{santhanam2012information,tandon2014information} and Gaussian Markov random fields \citep{wang2010information}.
Nevertheless, there is still a lack of understanding in the context of more general structured prediction problems.

Our main contribution consists of characterizing the necessary sample complexity for learning factor graph models in the context of structured prediction.
Specifically, in Theorem \ref{thrm:minimax_lb}, we show that the finiteness of the $\VCtwo$-dimension (see Definition \ref{def:pair_dimension}) is necessary for learning.
We further show in Theorem \ref{thrm:vc2_vc} the connection of the $\VCtwo$-dimension to the VC-dimension \citep{vapnik2013nature}, which will allow us to compute the $\VCtwo$-dimension from the several known results on VC-dimension.

\section{Preliminaries}
\label{sec:preliminaries}
	Let $\gX$ denote the input space and $\gY$ the output space.
	In structured prediction, the output space usually consists of a large (e.g., exponential) set of discrete objects admitting some possibly overlapping structure.
	Among common structures in the literature, one finds set of sequences, graphs, images, parse trees, etc.
	Thus, we consider the output space $\gY$ to be decomposable into $l$ substructures: $\gY = \gY_1 \times \dots \times \gY_l$.
	Here, $\gY_i$ is the set of possible labels that can be assigned to substructure $i$.
	For example, in a webpage collective classification task \citep{taskar2002discriminative}, each $\gY_i$ is a webpage label, whereas $\gY$ is a joint label for an entire website.
	In this work we assume that $\gY_i \in \{0, 1\}$, that is, $|\gY_i| = 2$ for all $i$.
	In this case, the number of possible assignments to $\gY$ is exponential in the number of substructures $l$, i.e., $|\gY| = 2^l$.
	
	\paragraph{The Hamming loss.}
	\label{sec:hamming_loss}
		In order to measure the success of a prediction, we use the Hamming loss throughout this work.
		Specifically, for two outputs $y,y' \in \gY$, with $y=(y_1,\dots,y_l)$ and $y' = (y'_1,\dots,y'_l)$, the Hamming loss, $L_H$, is defined as $L_H(y,y') =  \sum_{i=1}^l \Ind{y_i \neq y'_i}$.
		The Hamming loss has been widely used in structured prediction, for instance, in image segmentation one may count the number of pixels that are incorrectly assigned as foreground/background; in graphs, one may count the number of different edges between the prediction and the true label.
	
	\paragraph{Factor graphs and scoring functions.} 
	\label{sec:score_and_factors}
		We adopt a common approach in structured prediction where predictions are based on a scoring function mapping $\gX \times \gY$ to $\gR$. 
		Let $\gF$ be a family of scoring functions. 
		For any $f \in \gF$, we denote by $f(x)$ the predictor defined by $f$: for any $x \in \gX$ , $f(x) = \argmax_{y\in\gY} f(x, y).$
		
		Furthermore, we assume that each function $f \in \gF$ can be decomposed as a sum, as is standard in structured prediction. 
		We consider the most general case for such decompositions through the notion of factor graphs, described also in \citep{cortes2016structured}.
		A factor graph $G$ is a bipartite graph, and is represented as a tuple $G = (V, \Phi, E)$, where $V$ is a set of variable nodes, $\Phi$ a set of factor nodes, and $E$ a set of undirected edges between a variable node and a factor node. 
		In our context, $V$ can be identified with the set of substructure indices, that is $V = \{1,\dots,l\}.$
		We further assume that $G$ is connected.
		Note that, in contrast to graphical models, we do not assume $f$ to be a probabilistic model but it would also be captured by this framework.
		
		For any factor node $\phi \in \Phi$, denote by $\Scope(\phi) \subseteq V$ the set of variable nodes connected to $\phi$ via an edge and define $\gY_\phi$ as the substructure set cross-product $\gY_\phi = \bigtimes_{i\in \Scope(\phi)} \gY_i$. 
		Then, $f$ decomposes as a sum of functions $f_\phi$ , each taking as argument an element of the input space $x \in \gX$ and an element of $\gY_\phi$, $y_\phi \in \gY_\phi$:
		\[
			f(x,y) = \sum_{\phi \in \Phi} f_\phi (x, y_\phi).
		\]
		Specifically, we focus on factor graphs with unary and pairwise factors, that is, each factor node $\phi \in \Phi$ is connected to  one or two nodes in $V$.
		We let $\phi_{uv}$ denote a pairwise factor node connected to $u,v \in \gV$, i.e., $\Scope(\phi_{uv}) = \{u,v\}$.
		Then, the score induced from $\phi_{uv}$ is given by $f_{\phi_{uv}}(x,y_{\phi_{uv}}) = f_{\phi_{uv}}(x,y_u,y_v)$.
		Note that in this case $\phi_{uv}$ and $\phi_{vu}$ represent the same factor node and induce the same score.
		Similarly, for unary factor nodes, we let $\phi_{u}$ denote a factor node connected to $u \in \gV$ with score given by $f_{\phi_{u}}(x,y_{\phi_{u}}) = f_{\phi_{u}}(x,y_u)$.
		We further use $\gF(G)$ to denote functions that are decomposable with respect to the graph $G$.
		Note also that while all $f \in \gF(G)$ decompose with respect to same graph $G$, the score functions $f_\phi$ and $f'_\phi$ are allowed to be different for any $\phi \in \Phi$, $f,f' \in \gF(G)$. 
		Figure \ref{fig:graph_examples} shows different examples of factor graphs with unary and pairwise factors.
		
		\begin{figure}[ht]
			\begin{center}
				\begin{tikzpicture}[scale=1.0]
				\tikzset{>=latex}
				\tikzstyle{vertex}=[circle, fill=Kblue, draw, inner sep=0pt, minimum size=12pt]
				\tikzstyle{factor}=[rectangle, fill=Kgreen, draw, inner sep=0pt, minimum size=3pt]
					\node[vertex][label=center:\small$y_1$](y1) at (1,2) {};
					\node[vertex][label=center:\small$y_2$](y2) at (2,2) {};
					\node[vertex][label=center:\small$y_3$](y3) at (3,2) {};
					\node[factor][](f1) at (1,2.5) {};
					\node[factor][](f2) at (2,2.5) {};
					\node[factor][](f3) at (3,2.5) {};
					\node[factor][](f12) at (1.5,2) {};
					\node[factor][](f23) at (2.5,2) {};
				\tikzset{EdgeStyle/.style={-,color=Kred}}
					\Edge(y1)(f1)
					\Edge(y2)(f2)
					\Edge(y3)(f3)
					\Edge(f12)(y1)
					\Edge(f12)(y2)
					\Edge(f23)(y2)
					\Edge(f23)(y3)
				\end{tikzpicture}
				\hspace{0.25in}
				\begin{tikzpicture}[scale=1.0]
				\tikzset{>=latex}
				\tikzstyle{vertex}=[circle, fill=Kblue, draw, inner sep=0pt, minimum size=12pt]
				\tikzstyle{factor}=[rectangle, fill=Kgreen, draw, inner sep=0pt, minimum size=3pt]
					\node[vertex][label=center:\small$y_1$](y1) at (1,3) {};
					\node[vertex][label=center:\small$y_2$](y2) at (2,3) {};
					\node[vertex][label=center:\small$y_3$](y3) at (3,3) {};
					\node[vertex][label=center:\small$y_4$](y4) at (1,2) {};
					\node[vertex][label=center:\small$y_5$](y5) at (2,2) {};
					\node[vertex][label=center:\small$y_6$](y6) at (3,2) {};
					\node[factor][](f12) at (1.5,3) {};
					\node[factor][](f23) at (2.5,3) {};
					\node[factor][](f45) at (1.5,2) {};
					\node[factor][](f56) at (2.5,2) {};
					\node[factor][](f14) at (1,2.5) {};
					\node[factor][](f25) at (2,2.5) {};
					\node[factor][](f36) at (3,2.5) {};
				\tikzset{EdgeStyle/.style={-,color=Kred}}
					\Edge(f12)(y1); \Edge(f12)(y2);
					\Edge(f23)(y2); \Edge(f23)(y3);
					\Edge(f45)(y4); \Edge(f45)(y5);
					\Edge(f56)(y5); \Edge(f56)(y6);
					\Edge(f14)(y1); \Edge(f14)(y4);
					\Edge(f25)(y2); \Edge(f25)(y5);
					\Edge(f36)(y3); \Edge(f36)(y6);
				\end{tikzpicture}
				\hspace{0.25in}
				\begin{tikzpicture}[scale=1.0]
				\tikzset{>=latex}
				\tikzstyle{vertex}=[circle, fill=Kblue, draw, inner sep=0pt, minimum size=12pt]
				\tikzstyle{factor}=[rectangle, fill=Kgreen, draw, inner sep=0pt, minimum size=3pt]
					\node[vertex][label=center:\small$y_1$](y1) at (1,1) {};
					\node[vertex][label=center:\small$y_2$](y2) at (2,2) {};
					\node[vertex][label=center:\small$y_3$](y3) at (3,1) {};
					\node[vertex][label=center:\small$y_4$](y4) at (4,2) {};
					\node[vertex][label=center:\small$y_5$](y5) at (5,1) {};
					
					\node[factor][](f1) at (1.5,1) {};
					\node[factor][](f4) at (4.5,2) {};
					\node[factor][](f12) at (1.5,1.5) {};
					\node[factor][](f23) at (2.5,1.5) {};
					\node[factor][](f24) at (3,2) {};
					\node[factor][](f34) at (3.5,1.5) {};
					\node[factor][](f45) at (4.5,1.5) {};
				\tikzset{EdgeStyle/.style={-,color=Kred}}
					\Edge(f12)(y1); \Edge(f12)(y2);
					\Edge(f23)(y2); \Edge(f23)(y3);
					\Edge(f24)(y4); \Edge(f24)(y2);
					\Edge(f34)(y4); \Edge(f34)(y3);
					\Edge(f45)(y4); \Edge(f45)(y5);
					\Edge(f1)(y1); 
					\Edge(f4)(y4);
				\end{tikzpicture}
			\end{center}
			\caption{Examples of factor graphs with unary and pairwise factors.
					 \figleft Tree-structured factor graph.
					 \figcenter Grid-structured factor graph.
					 \figright Arbitrary factor graph with decomposition: $f(x,y) = f_{\phi_1}(x,y_1) + f_{\phi_4}(x,y_4) + f_{\phi_{12}}(x,y_1,y_2)+ f_{\phi_{23}}(x,y_2,y_3) + f_{\phi_{24}}(x,y_2,y_4) + f_{\phi_{34}}(x,y_3,y_4) + f_{\phi_{45}}(x,y_4,y_5)$
			}
			\label{fig:graph_examples}
		\end{figure}
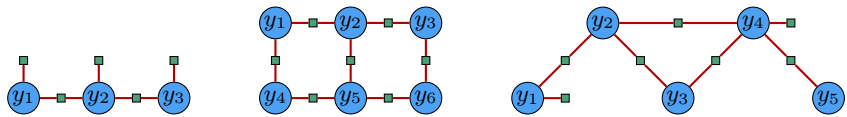

	\paragraph{Learning.}
	\label{sec:learning}
		We receive a training set $S = ((x_1,y_1),\dots,(x_m,y_m))$ of $m$ \iid samples drawn according to some distribution $P$ over $\gX \times \gY$.
		We denote by $R_P(f)$ the \textit{expected Hamming loss} and by $R_S(f)$ the \textit{empirical Hamming loss} of $f$:
		\begin{align}
		\label{eq:general_error}
			R_P (f) = \E_{(x,y) \sim P} [ L_H( f(x), y ) ] \quad \tand \quad R_S(f) = \frac{1}{m} \sum_{(x,y) \in S} L_H(f(x),y).
		\end{align}
		Our learning scenario consists of using the sample $S$ to select a hypothesis $f \in \gF(G)$ with small expected Hamming loss $R_P(f)$.

		Next, we introduce the definition of \textit{Bayes-Hamming loss}, which in words is the minimum attainable expected Hamming loss by any predictor.
		\begin{definition}[Bayes-Hamming loss]
		\label{def:Bayes_error}
			For any given distribution $P$ over $\gX \times \gY$, the Bayes-Hamming loss is defined as the minimum achievable expected Hamming loss among all possible predictors $f:\gX \to \gY$. 
			That is, $R^{*}=\min _{f} R_P(f).$
		\end{definition}
		Then the \textit{Bayes-Hamming predictor}, $f^*$, is defined as the function that achieves the Bayes-Hamming loss, that is, $R(f^*) = R^*$.
	
		The following proposition shows how the Bayes-Hamming predictor makes its decision with respect to the Hamming loss.
		\begin{proposition}
		\label{prop:bayes_classifier}
			For any given distribution $P$ over $\gX \times \gY$,  the Bayes-Hamming predictor $f^*$ is:
			$(f^{*}(x))_{i} = \Ind{\eta_i(x) \geq 1 / 2}.$
			where $\eta_i(x)$ is the marginal probability $\P[y_i=1|x]$, for each substructure $y_i$.
		\end{proposition}
		(See Appendix \ref{app:detailedproofs} for detailed proofs.)
		
		We emphasize that the above definition considers the Hamming loss, $L_H$, as defined at the beginning of Section \ref{sec:preliminaries}.
		For other types of loss functions, the Bayes predictor can have different optimal decisions.

	\subsection{Minimax risk framework}
	\label{sec:minimax_framework}
		The standard minimax risk consists of a family of distributions $\gP$ over a sample space $\gZ$, and a function $\theta : \gP \to \Theta$ defined on $\gP$, that is, a mapping $P \mapsto \theta(P)$.
		We aim to estimate the parameter $\theta(P)$ based on a sequence of \iid observations $(z_i)_{i=1}^m$ drawn from the (unknown) distribution $P$.
		To evaluate the quality of an estimator $\theta$, we let $\rho : \Theta \times \Theta \to \sR_+$ denote a semi-metric on the space $\Theta$, which we use to measure the error of an estimator $\thetah$ with respect to the parameter $\theta$.
		For a distribution $P \in \gP$ and for a given estimator $\thetah : \gZ^m \to \Theta$, we assess the quality of the estimate $\thetah(z_1, \ldots, z_m)$ in terms of the (expected) risk:
		\[
			\bE_P [ \rho(\thetah(z_1, \ldots, z_m), \theta(P)) ],
		\]
		where $\bE_P[ \cdot ]$ denotes the expectation with respect to $(z_1,\dots,z_m) \sim P^m$. 
		A common approach, first suggested by \citep{wald1939contributions}, for choosing an estimator $\thetah$ is to select the one that minimizes the maximum risk, that is,
		\[
			\sup_{P \in \gP} \bE_P [ \rho(\thetah(z_1, \ldots, z_m), \theta(P)) ].
		\]
		An optimal estimator for this metric then gives the minimax risk, which is defined as:
		\begin{align*}
			\Mfrak_{m}(\theta(\gP), \rho) := \inf_{\thetah} \sup_{P \in \gP} \bE_{P} \left[ \rho( \thetah (z_1, \ldots, z_m), \theta(P) ) \right],
		\end{align*}
		where we take the supremum (worst-case) over distributions $P \in \gP$, and the infimum is taken over all estimators $\thetah$.
		Here the notation $\theta(\gP)$ indicates that we consider distributions in $\gP$ and parameters $\theta(P)$ for $P \in \gP$.

	\subsection{Minimax risk in structured prediction}
	\label{sec:minimax_sp}
		We now apply the framework above to our context and study a specialized notion of risk appropriate for prediction problems.
		In this setting, we aim to estimate a function $f \in \gF(G)$ by using samples from a distribution $P$.
		For any sample $(x,y) \sim P$, we will measure the quality of our estimation, $f$, by comparing its output $f(x)$ to the structure $y$ drawn from $P$ through the Hamming loss.
		By taking expectation, we obtain the expected risk or expected Hamming loss, $R_P(f)$ defined in eq.\eqref{eq:general_error}.
		We then compare this risk to the best possible Hamming loss, i.e., the Bayes-Hamming loss.
		That is, we assume that at least one function $f \in \gF(G)$ achieves the Bayes-Hamming loss. 
		Thus, we arrive to the following \textit{minimax excess risk}:
		\begin{align}
		\label{eq:minimax}
			\Mfrak_m(\gP) = \inf_{\gA} \sup_{P \in \gP} \E_{S \sim P^m} [R_P(\gA(S)) - R_P(\fs)],
		\end{align}
		where $\fs = \argmin_{f\in \gF(G)} R_P(f)$, and $\gA: (\gX \times \gY)^m \to \gF(G)$ is any algorithm that returns a predictor given $m$ training samples from $P$.
		Moreover, $\gP$ defines a family of distributions over $\gX \times \gY$.
		Intuitively speaking, for a fixed distribution $P \in \gP$, the quantity $\Mfrak_m(\gP)$ represents the minimum expected excess loss achievable by any algorithm with respect to the factor graph $G$.
		Then $\Mfrak_m(\gP)$ looks into the distribution that attains the worst expected excess loss.
		

\section{Information-theoretic lower bound for structured prediction}
	We are interested on finding a lower bound to the minimax risk \eqref{eq:minimax} presented in Section \ref{sec:minimax_sp}.
	By doing this, we characterize the necessary number of samples to have any hope in achieving learning.
	
	Before presenting our main result, we introduce a new type of dimension that will show up in our lower bound and will help to characterize learnability.
	Note that it is known that different notions of dimension of function classes help to characterize learnability in certain prediction problems.
	For example, for binary classification, the finiteness of the VC dimension \citep{vapnik2013nature} is necessary for learning \citep{massart2006risk}.
	For multiclass classification, it was shown that the finiteness of the Natarajan dimension is necessary for learning \citep{daniely2015multiclass}.
	General notion of dimensions for multiclass classification has also been study in \citep{bendavid1995characterizations}.	
	
	For a given function class $\gG \subseteq \{g \ | \ g: \gX \to \{0,1\}^2 \}$, and dataset $S$ of $m$ samples, we use the following shorthand notation:
	\(
		\gG(S) = \{ (g(x_1), \ldots, g(x_m)) \in \{0,1\}^{m\times 2} \ | \ g \in \gG \}.
	\)
	That is, $\gG(S)$ contains all the matrices in $\{0,1\}^{m\times 2}$ that can be produced by applying all functions in $\gG$ to the dataset $S$.
	Next we define the standard notion of shattering.
	\begin{definition}[$\VCtwo$-shattering]
	\label{def:pair_shattering}
		A function class, $\gG$, $\VCtwo$-shatters a finite set $S$ of $m$ samples if $\gG(S)$ produces all possible binary matrices in $\{0,1\}^{m\times 2}$.
		That is, $|\gG(S)| = 2^{2m}$.
	\end{definition}
	
	\begin{definition}[$\VCtwo$-dimension]
	\label{def:pair_dimension}
		The $\VCtwo$-dimension of a function class $\gG$, denoted $\VCtwoDim(\gG)$, is the maximal size of a set $S$ that can be shattered by $\gG$. 
		If $\gG$ can shatter sets of arbitrarily large size we say that $\gG$ has infinite $\VCtwo$-dimension.
	\end{definition}
	The above dimension applies to functions with output  in $\{0,1\}^2$. 
	We will create functions with output in $\{0,1\}^2$ as follows.
	Let $ f^{(0)}_{u,v}(x,y_u,y_v) = f\left( x, \left( 0,\ldots,0, y_u, 0,\ldots,0, y_v, 0, \ldots,0 \right) \right)$ denote the function $f(x,y)$ with $y_i=0$ for all $i \in \{1,\ldots,l\}\setminus\{u,v\}$.
	Then, let $f^{(0)}_{u,v}(x) = \argmax_{y_u,y_v} f^{(0)}_{u,v}(x,y_u,y_v)$, that is, the output of $f^{(0)}_{u,v}(x)$ is in $\{0,1\}^2$.
	The following dimension applies to function classes based on factor graphs.
	\begin{definition}[$\MaxVCtwo$-dimension]
	\label{def:max_pair_dimension}
		For a given factor graph $G=(V,\Phi,E)$, the $\MaxVCtwo$-dimension of a function class $\gF(G)$, denoted as $\MaxVCtwoDim(\gF(G))$, is defined as:
		\[
			\MaxVCtwoDim(\gF(G))  =  \max_{(u,v) \in T} \ \VCtwoDim(\gF^{(0)}_{u,v}),
		\]
		where $T = \{(u,v) \in \Scope(\phi) \ |  \ \forall \phi \in \Phi \}$, and $\gF^{(0)}_{u,v} = \{ f^{(0)}_{u,v} \ |\ f \in \gF(G), (u,v) \in T  \}$.
	\end{definition}
	\begin{theorem}
	\label{thrm:minimax_lb}
		Let $G=(V,\Phi,E)$ be a factor graph with pairwise and unary factors, let $\gF(G)$ denote a class of functions $f: \gX \to \{0,1\}^l$, where each $f \in \gF(G)$ decomposes according to $G$, and let $d = \MaxVCtwoDim(\gF(G)) \geq 2$.
		Then, we have that for any $\gamma \in [0,\nicefrac{1}{3}]$ and any $m \geq d$:
		\[
			\Mfrak_m(\gP) \geq \frac{1}{81} \min \left( \frac{d -1}{\gamma m}, \sqrt{\frac{d -1}{m}} \right).
		\]
	\end{theorem}
	\begin{proof}
	The proof is motivated by the work of \cite{massart2006risk} for binary classifiers.
	As a first step it is clear that one can lower bound eq.\eqref{eq:minimax} by defining the maximum over a subset of $\gP$.
	That is, we create a collection of family of distributions $\sD_\gamma$, where $|\sD| = |\Phi|$.
	Each family distribution $\gD_{\gamma,u,v} \in \sD_\gamma$ is further indexed by $(u,v) \in T = \{(u,v) \in \Scope(\phi) \ |  \ \forall \phi \in \Phi  \}$.
	Then we have,
	\begin{align*}
	\Mfrak_m(\gP) \geq \max_{(u,v) \in T} \Mfrak_m(\gD_{\gamma,u,v}).
	\end{align*}
	Our approach consists of first defining the families of distributions $\gD_{\gamma,u,v} \subset \gP$ such that its elements can be naturally indexed by the vertices of a binary hypercube.
	We will then relate the expected excess risk problem to an estimation of binary strings in order to apply Assouad's lemma.
	
	\paragraph{Construction of $\gD_{\gamma,u,v}$.} 
	Consider  a fixed $(u,v) \in T$.
	We first focus on constructing a family of distributions, $\gD_{\gamma,u,v}$, parameterized by $\gamma > 0$.
	Each distribution $D_{\gamma,u,v,B} \in \gD_{\gamma,u,v}$ is further indexed by a binary matrix $B \in \{0,1\}^{(\duv -1)\times 2}$, where $\duv$ is the $\VCtwo$-dimension of $\gF^{(0)}_{u,v}$.
	To construct these distributions, we will first pick the marginal distribution $D^{(x)}_{\gamma,u,v,B}$ of the feature $x$, and then specify the conditional distributions $D^{(y|x)}_{\gamma,u,v,B}$ of $y$ given $x$, for each $B \in \{0,1\}^{(\duv - 1)\times 2}$.
	
	We construct $D_{\gamma,u,v,B}^{(x)}$ as follows. 
	Since $\gF^{(0)}_{u,v}$ is a class with $\VCtwo$-dimension $\duv$, there exists a set of points $\{x_1,\ldots,x_{\duv}\} \in \gX$ that are shattered by $\gF^{(0)}_{u,v}$, that is, for any binary matrix $B \in \{0,1\}^{\duv \times 2}$ there exists at least one function $f^{(0)}_{u,v} \in \gF^{(0)}_{u,v}$ such that $f^{(0)}_{u,v}(x_i) = B_{i*}$, for all $i \in \{1,\ldots,\duv\}$.
	We now define the marginal distribution $D_{\gamma,u,v,B}^{(x)}$ such that its support is the shattered set $\{x_1,\ldots,x_{\duv}\}$, i.e., $\P_{\gamma,u,v,B}^{(x)}[\{x_1,\ldots,x_{\duv} \}] = 1$.
	For a given parameter $p \in [0, \nicefrac{1}{(\duv-1)} ]$, whose value is set later, we have:
	\begin{align*}
	\bP_{\gamma,u,v,B}^{(x)} [x_i] = \left\{ \begin{array}{ll} 		  
			p, & \text{if } i \in \{1,\ldots, \duv -1\} \\ 
			1 - (\duv-1) p, & \text{otherwise.} 
	\end{array} \right.
	\end{align*}
	Next, for a fixed $B \in \{0,1\}^{(\duv-1)\times 2}$, the conditional distribution of $y$ given $x$, $D_{\gamma,u,v,B}^{(y|x)}$, is defined as:
	\begin{align*}
	\bP_{\gamma,u,v,B}^{(y|x)}[y|x]  = \left\{\begin{array}{ll} 
			\frac{1 - 3 \gamma}{4},		& \text{if } x=x_i,\ \tand \ y_u = 1 - B_{i1}, \ \tand \ y_v = 1 - B_{i2}, \\
										& \tand \ y_k=0 \text{ for } k \in V\setminus\{u,v\}, \ \tand \ i \in \{1,\ldots,\duv-1\} \\ 
			\frac{1 +   \gamma}{4},		& \text{if } x=x_i, \ \tand \ y_k=0 \text{ for } k \in V\setminus\{u,v\}, \ \tand \ i \in \{1,\ldots,\duv-1\}\\ 
			0,							& \text{otherwise,}
	\end{array}\right.
	\end{align*}
	here we implicitly assume that $\gamma \in (0,1/3]$ in order to obtain a valid distribution.
	The above definition produces the following marginal probabilities:
	\begin{align}
	\label{eq:eta_marginal}
	\eta_{j}^{(\gamma,u,v,B)} (x) \equiv \bP^{(y_j | x)}_{\gamma,u,v,B} [y_j=1 | x] = \left\{\begin{array}{ll}
			{\frac{1-\gamma}{2},} & \text{if } x=x_{i} \text { for some } i \in\{1, \ldots, \duv-1\}, \tand\\
			 					  &  ( (j=u \ \tand \ B_{i1}=0)\ \tor \ (j=v \ \tand \ B_{i2}=0)) \\ 
			{\frac{1+\gamma}{2},} & \text{if } x=x_{i} \text { for some } i \in\{1, \ldots, \duv-1\}, \tand\\
								  &  ( (j=u \ \tand \ B_{i1}=1)\ \tor \ (j=v \ \tand \ B_{i2}=1)) \\ 
			{0,} & {\text{otherwise,}}
	\end{array}\right.
	\end{align}
	where we note that for each $j \in V$ and any $x$ we have that $|2\eta_{j}^{(\gamma,u,v,B)} (x) - 1| \geq \gamma$.
	Given the above marginals, the corresponding Bayes-Hamming predictor for substructure $y_j$ for a given input $x$ (see Proposition \ref{prop:bayes_classifier}), which we denote by $(f^*_{B,u,v}(x))_j$, is given by:
	\begin{align}
	\label{eq:fs_B}
	(f_{B,u,v}^{*}(x))_j = \left\{\begin{array}{ll}
		0, 	& \text { if } x=x_{i} \text { for some } i \in\{1, \ldots, \duv-1\}, \\
			& \tand \ ( (j=u \ \tand \ B_{i1}=0)\ \tor \ (j=v \ \tand \ B_{i2}=0)) \\ 
		1,  & \text { if } x=x_{i} \text { for some } i \in\{1, \ldots, \duv-1\}, \\
			& \tand \ ( (j=u \ \tand \ B_{i1}=1)\ \tor \ (j=v \ \tand \ B_{i2}=1)) \\ 
		0,  & \text { otherwise.}
	 \end{array}\right.
	\end{align}
	That is, we have that the output of the Bayes-Hamming predictor on each $x_i$ for $i \in \{1\ldots \duv-1\}$, for each substructure $y_j$ for $j \in \{u,v\}$, is equal to the bit value $B_{i1}$ or $B_{i2}$, and zero otherwise.

	\paragraph{Reduction to estimation of binary strings.}
	For any distribution $D_{\gamma,u,v,B} \in \gD_{\gamma,u,v}$, we can further express the expected excess risk in eq.\eqref{eq:minimax} as follows:
	\begin{flalign}
	&R_{B,u,v}(\gA(S)) - R_{B,u,v}(\fs_{B,u,v})
	=  \E_{(x,y) \sim D_{\gamma,u,v,B}} \left[ \sum_{j=1}^l \left( 1 - 2 y_j \right) \left( (\fh_m(x))_j - (\fs_{B,u,v}(x))_j \right) \right] \nonumber \\
	&=  \sum_{j=1}^l \E_{x \sim D_{\gamma,u,v,B}^{(x)}} \left[ \E_{y_j \sim D_{\gamma,u,v,B}^{(y_j|x)} }  \left[ \left( 1 - 2 y_j \right) \left( (\fh_m(x))_j - (\fs_{B,u,v}(x))_j \right) \right] \right] \nonumber \\
	&=  \sum_{j=1}^l \E_{x \sim D_{\gamma,u,v,B}^{(x)}} \left[  \left| 2 \eta^{(\gamma,u,v,B)}_j(x) - 1 \right| \cdot  \left| (\fh_m(x))_j - (\fs_{B,u,v}(x))_j \right| \right] \nonumber \\
	&\geq \gamma \cdot \E_{x \sim D_{\gamma,u,v,B}^{(x)}} \left[ \sum_{j=1}^l \left| (\fh_m(x))_j - (\fs_{B,u,v}(x))_j \right| \right] \label{subeq:gamma} \\
	&= \gamma \cdot \sum_{i=1}^{\duv} \sum_{j=1}^l \left| (\fh_m(x_i))_j - (\fs_{B,u,v}(x_i))_j \right| \cdot \bP^{(x)}_{\gamma,u,v,B} [x_i] 
	\defeq \gamma \cdot \Norm{ \fh_m - \fs_{B,u,v} }_{1,1}, \label{subeq:l1}
	\end{flalign}
	where $R_{B,u,v}$ denotes the expected risk and $\fs_{B,u,v}$ the Bayes-Hamming predictor, both with respect to $D_{\gamma,u,v,B}$. 
	Here $\fh_m$ is the output of $\gA(S)$, with $(\fh_m(x))_j$ denoting the $j$-th substructure of the output $\fh_m(x)$, and $\eta^{(\gamma,u,v,B)}_j(x)$ denotes the marginal probability $\bP_{D^{(y_j|x)}_{\gamma,u,v,B} } [y_j=1 | x]$.
	Equation \eqref{subeq:gamma} follows from our definition of $D^{(y_j|x)}_{\gamma,u,v,B}$ (see eq.\eqref{eq:eta_marginal}), and the $L_{1,1}$ matrix norm in eq.\eqref{subeq:l1} is computed with respect to $D_{\gamma,u,v,B}^{(x)}$.
	Thus, we have that:
	\begin{align}
	\Mfrak_m(\gD_{\gamma,u,v}) 
	&= \inf_{\fh_m} \max_{B \in \{0, 1\}^{(\duv-1)\times 2}} \bE_{B,u,v} \left[ R_{B,u,v}(\fh_m) - R_{B,u,v}(\fs_{B,u,v}) \right] \nonumber \\
	&\geq \gamma \cdot \inf_{\fh_m} \max_{B \in \{0, 1\}^{(\duv-1)\times 2}}  \bE_{B,u,v} \left[ \Norm{\fh_m - \fs_{B,u,v} }_{1,1} \right], \label{subeq:minimax_B}
	\end{align}
	where $\bE_{B,u,v}[\cdot]$ denotes the expectation with respect to $S \sim D^m_{\gamma,u,v,B}$.
	Equation \eqref{subeq:minimax_B} follows from eq.\eqref{subeq:l1}.
	Given any candidate estimation $\fh_m$, let $\Bh_m \in \{0, 1\}^{(\duv-1)\times 2}$ be defined as follows:
	\begin{align}
		\Bh_m &\defeq \argmin_{B \in \{0, 1\}^{(\duv-1)\times 2}} \Norm{ \fh_m - \fs_{B,u,v} }_{1,1}. \label{eq:Bh_m_def}
	\end{align}
	Intuitively, $\Bh_m$ is the binary matrix that indexes the element of  $\{\fs_{B,u,v} : B \in \{0, 1\}^{(\duv-1)\times 2} \}$ which is the closest to $\fh_m$ in  $L_{1,1}$ norm. 
	Then, for any $B$, we have
	\begin{align*}
		\Norm{ \fs_{\Bh_m,u,v} - \fs_{B,u,v} }_{1,1}		&\leq	\Norm{ \fs_{\Bh_m,u,v} - \fh_m }_{1,1}		+		\Norm{ \fh_m - \fs_{B,u,v} }_{1,1} 
												\leq 2 \Norm{ \fh_m - \fs_{B,u,v} }_{1,1},
	\end{align*}
	where we first applied the triangle inequality, and then used eq.\eqref{eq:Bh_m_def}.
	Applying this to eq.\eqref{subeq:minimax_B}, we obtain:
	\begin{align}
		\label{eq:estimation_fs}
		\Mfrak_m(\gD_{\gamma,u,v}) \geq \frac{\gamma}{2} \inf_{\Bh_m} \max_{B \in \{0, 1\}^{(\duv-1)\times 2}} \bE_{B,u,v} \left[ \Norm{ \fs_{\Bh_m,u,v} - \fs_{B,u,v} }_{1,1} \right], 
	\end{align}
	here the infimum is  over all estimators that take values in $\{0,1\}^{(\duv-1)\times 2}$ based on $m$ samples, i.e., over $\Bh_m : (\gX \times \gY)^m \to \{0,1\}^{(\duv-1) \times 2}$.
	We now compute $\Norm{ \fs_{B,u,v} - \fs_{B',u,v} }_{1,1}$ for any two $B, B'$.
	Using eq.\eqref{eq:fs_B} we have:
	\begin{align*}
				\Norm{ \fs_{B,u,v} - \fs_{B',u,v} }_{1,1} &= \sum_{i=1}^{\duv} \sum_{j=1}^l \left| (\fs_{B,u,v}(x_i))_j - (\fs_{B',u,v}(x_i))_j \right| \cdot \bP^{(x)}_{\gamma,u,v,B} [x_i] \\
												  &= p \cdot \sum_{i=1}^{\duv-1} \sum_{j=1}^2 \ \left|	B_{ij} - B'_{ij}	\right| 
												  = p\cdot L_H (B,B').
	\end{align*}
	In the last equality we abuse notation and consider the matrix $B \in \{0,1\}^{(\duv-1)\times 2}$ as a vector of dimension $2(\duv-1)$.
	Replacing this result into eq.\eqref{eq:estimation_fs}, we get:
	\begin{align*}
		\Mfrak_m(\gD_{\gamma,u,v}) 
		\geq \frac{p\gamma}{2} \inf_{\Bh_m} \max_{B \in \{0, 1\}^{(\duv-1)\times 2}} \bE_{B,u,v} \left[ L_H (\Bh_m, B) \right],
	\end{align*}
	which is related to an estimation problem in the $\{0,1\}^{2(\duv-1)}$ hypercube.
	\paragraph{Applying Assouad's lemma.}
	In order to apply Assouad’s lemma, we need an upper bound on the squared Hellinger distance $H^2( D_{\gamma,u,v,B}, D_{\gamma,u,v,B'})$ for all $B,B'$ with $L_H(B,B')=1$.
	For any two $B,B' \in \{0,1\}^{(\duv-1)\times 2}$ we have:
	\begin{align*}
		H^2 (D_{\gamma,u,v,B}, D_{\gamma,u,v,B'}) 	
		&=	\sum_{i=1}^{\duv} \sum_{y \in \{0,1\}^l} \left(	\sqrt{\bP_{\gamma,u,v,B} (x_i,y) }  -	 \sqrt{\bP_{\gamma,u,v,B'} (x_i,y) }	\right)^2\\
		&=	p \cdot \sum_{i=1}^{\duv-1} \sum_{y \in \{0,1\}^l} \left(	\sqrt{\bP^{(y|x)}_{\gamma,u,v,B'} (y|x_i) }  -  \sqrt{\bP^{(y|x)}_{\gamma,u,v,B} (y|x_i) } \right)^2.
	\end{align*}
	In the above summation, the inner sum is zero if $B_{i*} = B'_{i*}$.
	Since we are interested on $B$ and $B'$ such that $L_H(B,B') = 1$, this implies that for only one row $i$ from $\{1,\ldots,\duv - 1\}$ we have $B_{i*} \neq B'_{i*}$ with exactly one bit different.
	Then, the Hellinger distance results in:
	$
		H^2 ( D_{\gamma,u,v,B},  D_{\gamma,u,v,B'}) 
		=	p \cdot (1-\gamma - \sqrt{1 - 2\gamma - 3\gamma^2}) 
		\leq 6p\gamma^2.
	$
	Applying Assouad's lemma we obtain:
	\begin{align}
		\hspace{-0.15in}
		\Mfrak_m(\gD_{\gamma,u,v}) 
		\geq\frac{p\gamma}{2} \inf_{\Bh_m} \max_{B \in \{0, 1\}^{(\duv-1)\times 2}} \hspace{-0.2in}\bE_{B,u,v} \left[ L_H (\Bh_m, B) \right] 
		\geq \frac{p\gamma(\duv-1)}{2} (1- \sqrt{6p\gamma^2m}) \label{subeq:replace_p}
	\end{align}
	Let $p=\nicefrac{2}{(27\gamma^2m)}$, and noting that if $\gamma \geq \sqrt{(\duv-1)/m}$ then the condition $p \leq \nicefrac{1}{(\duv-1)}$ holds.
	Replacing $p$ in eq.\eqref{subeq:replace_p} we have:
	\begin{align}
	\label{eq:1st_bound}
		\Mfrak_m(\gD_{\gamma,u,v}) \geq \frac{\duv -1}{81 \gamma m}.
	\end{align}
	If $\gamma \leq \sqrt{(\duv-1)/m}$, and using the same construction as above with $\widetilde{\gamma}=  \sqrt{(\duv-1)/m}$, we see that:
	\begin{align}
	\label{eq:2nd_bound}
	\Mfrak_m(\gD_{\gamma,u,v}) \geq \frac{\duv -1}{81 \widetilde{\gamma} m} = \frac{1}{81} \sqrt{\frac{\duv -1}{m}}.
	\end{align}
	Therefore, combining equations \eqref{eq:1st_bound} and \eqref{eq:2nd_bound}, and
	since the choice of $(u,v)$ was arbitrary, we have that:
	\begin{align*}
		\Mfrak_m(\gP) 
		&\geq \max_{(u,v) \in T} \Mfrak_m(\gD_{\gamma},u,v) \geq \max_{(u,v) \in T} \frac{1}{81}\min \left( \frac{\duv -1}{\gamma m}, \sqrt{\frac{\duv -1}{m}} \right) \\
		&= \frac{1}{81} \min \left( \frac{\MaxVCtwoDim(\gF) -1}{\gamma m}, \sqrt{\frac{\MaxVCtwoDim(\gF) -1}{m}} \right).
	\end{align*}
	\end{proof}
	
\section{Relation of $\VCtwo$-dimension to VC-dimension}

	In this section, we show a connection of our defined $\VCtwo$-dimension to the classical VC-dimension \citep{vapnik2013nature}.
	
	The following theorem shows that for a function class $\gG:\gX \to \{0,1\}^2$, the $\VCtwo$-dimension of $\gG$ is related to the minimum VC-dimension of a subclass of functions derived from $\gG$.
	\begin{theorem}
	\label{thrm:vc2_vc}
		Let $\gG \subseteq \{g \ | \ g: \gX \to \{0,1\}^2 \}$ be a function class.
		Let $\gH_{11}, \gH_{10}, \gH_{01}, \gH_{00} \subseteq {\{h \ | \ h: \gX \to \{0,1\} \}}$ be four function classes defined as
		\begin{align*}
			\gH_{11} = \{ h(\cdot) = g(\cdot)_1 g(\cdot)_2 \mid g \in \gG\} \; ,
			& \;\; \gH_{10} = \{ h(\cdot) = g(\cdot)_1 (1-g(\cdot)_2) \mid g \in \gG\} \; , \\
			\gH_{01} = \{ h(\cdot) = (1-g(\cdot)_1) g(\cdot)_2 \mid g \in \gG\} \; ,
			& \;\; \gH_{00} = \{ h(\cdot) = (1-g(\cdot)_1) (1-g(\cdot)_2) \mid g \in \gG\} \; .
		\end{align*}
		We have that $\VCtwoDim(\gG) = \min(\VCDim(\gH_{11}), \VCDim(\gH_{10}), \VCDim(\gH_{01}), \VCDim(\gH_{00}))$.
	\end{theorem}
%
%

	\section{Discussion}
		We consider the problem of finding the necessary number of samples for learning of factor graphs with unary a pairwise factors.
		Our work was based on the minimax framework, that is, in obtaining a lower bound to the minimax risk.
		We showed a lower bound that requires the $\MaxVCtwo$-dimension to be finite in order for a function class to be learnable.
		We also note that in the proof of Theorem \ref{thrm:minimax_lb}, our choice of setting a value of zero to many $y$'s was for clarity purposes.
		In principle, one can create such distributions by fixing $y$'s to arbitrary values in $\{0,1\}^{l-2}$, and this would result in a slightly different notion of dimension, which would take the maximum across the $2^{l-2}$ different values.
		However, our focus was on providing a clear guideline to obtain lower bounds in structured prediction, hence, we opted for simplicity.
		In addition, in Theorem \ref{thrm:vc2_vc} we showed the connection of the $\VCtwo$-dimension to the VC-dimension, for which there are several known results for different types of function classes.
		
		An interesting future work is the analysis of tightness.
		For example, regarding tightness for linear classifiers, consider inputs $x \in \gR^k$.
		We observe that our lower bound in Theorem \ref{thrm:minimax_lb} is tight with respect to $k$ and $m$.
		Specifically, consider non-sparse linear classifiers as unary and pairwise potentials, Theorem 2 in \citep{cortes2016structured} gives $\gO(\sqrt{\nicefrac{k}{m}})$.
		In this case, the $\VCtwo$-dimension is equal to the VC-dimension, and the latter is equal to $k$.
		Thus, we obtain a lower bound with rate $\sqrt{\nicefrac{k}{m}}$ for some $\gamma$.
		Similarly, consider sparse linear classifiers as unary and pairwise potentials.
		Then, Theorem 2 of \citep{cortes2016structured} gives $\gO(\sqrt{\nicefrac{\log k}{m}})$.
		In this case, the VC-dimension is $\gO(\log k)$ \citep{Neylon06}, thus, we obtain a lower bound with rate $\sqrt{\nicefrac{\log k}{m}}$ for some $\gamma$.
		However, it remains to analyze for general functions where one possible attempt is perhaps to find an upper bound to the \textit{factor graph Rademacher complexity} \citep{cortes2016structured} in terms of the $\VCtwo$-dimension, similar in spirit to the known result of the VC-dimension being an upper bound of the classical Rademacher complexity (see for instance, \citep{shalev2014understanding}).

	\bibliographystyle{agsm}
	\bibliography{minimax_SP}
	

\newpage
\appendix
\normalsize
\onecolumn
\def\toptitlebar{
	\hrule height4pt
	\vskip .25in}

\def\bottomtitlebar{
	\vskip .25in
	\hrule height1pt
	\vskip .25in}

\thispagestyle{empty}
\hsize\textwidth
\linewidth\hsize \toptitlebar {\centering
{\large\bf SUPPLEMENTARY MATERIAL \\  Minimax bounds for structured prediction \par}}
\vspace{-0.1in} \bottomtitlebar

\section{Detailed Proofs} \label{app:detailedproofs}

\subsection{Proof of Proposition \ref{prop:bayes_classifier}}

	\begin{proof}
		Recall that $\eta_i(x) = \P[y_i=1|x]$.
		From eq.\eqref{eq:general_error} and Definition \ref{def:Bayes_error}, the Bayes-Hamming predictor $f^*$ minimizes the following expression (with respect to $f$)
		\begin{align*}
			R_P(f) & = \E_{(x,y) \sim P} [ L_H( f(x), y ) ] \\
				& = \E_{(x,y) \sim P} \left[ \sum_{i=1}^l \Ind{(f(x))_i \neq y_i} \right] \\
				& = \sum_{i=1}^l \E_{(x,y) \sim P} [ \Ind{(f(x))_i \neq y_i} ] \\
				& = \sum_{i=1}^l \E_x [ \P[y_i=1|x] (1-(f(x))_i) + (1-\P[y_i=1|x]) (f(x))_i ] \\
				& = \sum_{i=1}^l \E_x [ \eta_i(x) (1-(f(x))_i) + (1-\eta_i(x)) (f(x))_i ] \; .
		\end{align*}
		In order to minimize the above expression, for any $x$ we choose $(f(x))_i=1$ if $\eta_i(x) \geq 1/2$, and $(f(x))_i=0$ otherwise.
	\end{proof}

\subsection{Proof of Theorem \ref{thrm:vc2_vc}}

	\begin{proof}
		Recall that for a dataset $S$ of $m$ samples, $\gG(S) = \{ (g(x_1), \ldots, g(x_m)) \in \{0,1\}^{m\times 2} \ | \ g \in \gG \}$.
		Similarly, define $\gH_{ij}(S) = \{ (h(x_1), \ldots, h(x_m)) \in \{0,1\}^m \ | \ h \in \gH_{ij} \}$ for all $i,j \in \{0,1\}$.		
		Let $\VCtwoDim(\gG) = d$.
		
		There exists a dataset $S$ of $d$ samples such that $|\gG(S)|=2^{2d}$.
		Thus for all $i,j \in \{0,1\}$ we have $|\gH_{ij}(S)|=2^d$, which implies that for all $i,j \in \{0,1\}$ we have $\VCDim(\gH_{ij}) \geq d$.
		Therefore $\min(\VCDim(\gH_{11}), \VCDim(\gH_{10}), \VCDim(\gH_{01}), \VCDim(\gH_{00})) \geq d$.
		
		Also, for any dataset $S$ of $d+1$ samples we have $|\gG(S)|<2^{2(d+1)}$.
		Thus there exists $i,j \in \{0,1\}$ such that $|\gH_{ij}(S)|<2^{d+1}$, implying that there exists $i,j \in \{0,1\}$ such that $\VCDim(\gH_{ij}) < d+1$.
		Therefore $\min(\VCDim(\gH_{11}), \VCDim(\gH_{10}), \VCDim(\gH_{01}), \VCDim(\gH_{00})) < d+1$.
		
		From the above, $\min(\VCDim(\gH_{11}), \VCDim(\gH_{10}), \VCDim(\gH_{01}), \VCDim(\gH_{00})) = d$.
	\end{proof}

\end{document}